\documentclass[letterpaper,11pt]{article}

\pdfoutput=1     

\usepackage{fullpage}

\usepackage{graphicx} 
\usepackage{subfigure}

\usepackage[square,comma,numbers,sort&compress,sectionbib]{natbib}

\usepackage{algorithm}
\usepackage{algorithmic}

\usepackage{color}
\usepackage{hyperref}
\usepackage{hypernat}

\usepackage{amsmath}
\usepackage{amsthm}
\usepackage{amssymb}

\usepackage{array}

\theoremstyle{plain}   
\newtheorem{proposition}{Proposition}
\newtheorem{theorem}{Theorem}
\newtheorem{lemma}{Lemma}

\theoremstyle{definition}   
\newtheorem{mydef}{Definition}
\newtheorem{remark}{Remark}

\title{Taming the Curse of Dimensionality:\\
  Discrete Integration by Hashing and Optimization}

\author{Stefano Ermon, Carla P.~Gomes\\
  Dept.~of Computer Science\\
  Cornell University, Ithaca NY 14853, U.S.A.\\
  \texttt{\{ermonste,gomes\}@cs.cornell.edu}
  \and
  Ashish Sabharwal\\
  IBM Watson Research Center\\
  Yorktown Heights, NY 10598, U.S.A.\\
  \texttt{ashish.sabharwal@us.ibm.com}
  \and
  Bart Selman\\
  Dept.~of Computer Science,\\
  Cornell University, Ithaca NY 14853, U.S.A.\\
  \texttt{selman@cs.cornell.edu}
  }

\date{February 15, 2013}

\begin{document}

\maketitle

\begin{abstract}
Integration is affected by the curse of dimensionality and quickly
becomes intractable as the dimensionality of the problem grows.  We
propose a randomized algorithm that, with high probability, gives a
constant-factor approximation of a general discrete integral defined
over an exponentially large set. This algorithm relies on solving only
a small number of instances of a discrete combinatorial optimization
problem subject to randomly generated parity constraints used as a
hash function.  As an application, we demonstrate that with a small
number of MAP queries we can efficiently approximate the partition
function of discrete graphical models, which can in turn be used, for
instance, for marginal computation or model selection.
\end{abstract}

\section{Introduction}
Computing integrals in very high dimensional spaces is a fundamental and largely unsolved problem of scientific computation~\cite{dyer1991random,simonovits2003compute,cai2010decidable}, with numerous applications ranging from machine learning and statistics to biology and physics. As the volume grows exponentially in the dimensionality, the problem quickly becomes computationally intractable, a phenomenon traditionally known as the \emph{curse of dimensionality}~\cite{bellman1961adaptive}. 

We revisit the problem of approximately computing discrete integrals, namely weighted sums over (extremely large) sets of items. This problem encompasses several important probabilistic inference tasks, such as computing marginals or normalization constants (partition function) in graphical models, which are in turn the cornerstones for parameter and structure learning~\cite{wainwright2008graphical}. Although we focus on the discrete case, the continuous case can in principle also be addressed, as it can be approximated by numerical integration. There are two common approaches to approximate these large discrete sums: sampling and variational methods. Variational methods~\cite{wainwright2008graphical,jordan1999introduction}, often inspired by statistical physics, are very fast but do not provide guarantees on the quality of the results. Since sampling and counting can be reduced to each other~\cite{jerrum1997markov}, approximate  techniques based on sampling are quite popular, but they suffer from similar issues because the number of samples required to obtain a statistically reliable estimate often grows exponentially in the problem size. Among sampling techniques, Markov Chain Monte Carlo (MCMC) methods are asymptotically accurate, but guarantees for practical applications exist only in a limited number of cases (fast mixing chains)~\cite{jerrum1997markov,madras2002lectures}. They are therefore often used in an heuristic manner. In practice, their performance crucially depends on the choice of the proposal distributions, which often must be domain-specific and expert-designed~\cite{girolami2011riemann,murray2004bayesian}.

We introduce a randomized scheme that computes with high probability ($1 - \delta$ for any desired $\delta > 0$) an approximately correct estimate (within a factor of $1 + \epsilon$ for any desired $\epsilon > 0$) for general weighted sums defined over exponentially large sets of items, such as the set of all possible variable assignments in a discrete probabilistic graphical model. From a computational complexity perspective, the counting problem we consider is complete for the \#P complexity class~\cite{valiant1979complexity}, a set of problems encapsulating the entire Polynomial Hierarchy and believed to be significantly harder than NP.

The key idea is to reduce this \#P problem to a small number (polynomial in the dimensionality) of instances of a (NP-hard) combinatorial optimization problem defined on the same space and subject to randomly generated ``parity'' constraints.
The rationale behind this approach is that although combinatorial optimization is intractable in the worst case, it has witnessed great success in the past 50 years in fields such as Mixed Integer Programming (MIP) and propositional Satisfiability Testing (SAT). Problems such as computing a Maximum a Posteriori (MAP) assignment, although NP-hard, can in practice often be approximated~\cite{DBLP:conf/uai/SontagMGJW08} or solved exactly fairly efficiently~\cite{Park:2002:MCR:2073876.2073922,park2002using}. In fact, modern solvers can exploit structure in real-world problems and prune large portions of the search space, often dramatically reducing the runtime. In contrast, in a \#P counting problem such as computing a marginal probability, one needs to consider contributions of an exponentially large number of items.

Our algorithm, called \textbf{W}eighted-\textbf{I}ntegrals-And-\textbf{S}ums-By-\textbf{H}ashing (\texttt{WISH}), relies on randomized hashing techniques to ``evenly cut'' a high dimensional space. Such hashing was introduced by \citet{valiant1986np} to study the relationship between the number of solutions and the hardness of a combinatorial search. These techniques were also applied by \citet{gomes2006near,mbound} to obtain bounds on the number of solutions for the SAT problem. Our work is more general in that it can handle general weighted sums, such as the ones arising in probabilistic inference for graphical models. Our work is also closely related to recent work by \citet{hazan2012partition}, who obtain a lower bound on the partition function
by taking suitable expectations of a combination of MAP queries over randomly perturbed models.
We improve upon this in two crucial aspects, namely, our estimate is a constant factor approximation of the true partition function (while their bounds have no tightness guarantee), and we provide a concentration result showing that our bounds hold not just in expectation but with high probability with a polynomial number of MAP queries. Note that this is consistent with known complexity results regarding \#P and BPP$^\mathrm{NP}$; see Remark 1 below.

We demonstrate the practical efficacy of the \texttt{WISH} algorithm in the context of computing the partition function of random Clique-structured Ising models, Grid Ising models with known ground truth, and a challenging combinatorial application (Sudoku puzzle) completely out of reach of techniques such as Mean Field and Belief Propagation. We also consider the Model Selection problem in graphical models, specifically in the context of hand-written digit recognition. We show that our ``anytime'' and highly parallelizable algorithm can handle these problems at a level of accuracy and scale well beyond the current state of the art.

\section{Problem Statement and Assumptions}
Let $\Sigma$ be a (large) set of items. Let $w:\Sigma \rightarrow \mathbb{R}^+$ be a non-negative function that assigns a weight to each element of $\Sigma$. We wish to (approximately) compute the total weight of the set, defined as the following discrete integral or ``partition function''
\begin{equation}
\label{Wdef}
W = \sum_{\sigma \in \Sigma} w(\sigma)
\end{equation}
We assume $w$ is given as input and that it can be compactly represented, for instance in a factored form as the product of conditional probabilities tables. Note however that our results are more general and do not rely on a factored representation.

\textbf{Assumption:} We assume to have access to an \emph{optimization oracle} that can solve the following constrained optimization problem
\begin{equation}
\max_{\sigma \in \Sigma} w(\sigma) 1_{\{\mathcal{C}\}} (\sigma)
\end{equation}
where $1_{\{\mathcal{C}\}} : \Sigma \to \{0,1\}$ is an indicator function for a compactly represented subset $\mathcal{C} \subseteq \Sigma$, i.e., $1_{\{\mathcal{C}\}} (\sigma) = 1$ iff $\sigma \in \mathcal{C}$.
For concreteness, we discuss our setup and assumptions in the context probabilistic graphical models, which is our motivating application.

\subsection{Inference in Graphical Models}
We consider a graphical model specified as a factor graph with $N=|V|$ discrete random variables $x_i, i \in V$ where $x_i \in \mathcal{X}_i$. The global random vector $x = \{x_s, s \in V\}$ takes value in the cartesian product $\mathcal{X} = \mathcal{X}_1 \times \mathcal{X}_2 \times \cdots \times \mathcal{X}_N$.
We consider a probability distribution over $x \in \mathcal{X}$ (called \textbf{configurations})
$p(x) = \frac{1}{Z} \prod_{\alpha \in \mathcal{I}} \psi_\alpha(\{x\}_\alpha)$
that factors into potentials or factors $\psi_\alpha:\{x\}_\alpha \mapsto \mathbb{R}^+$, where $\mathcal{I}$ is an index set and $\{x\}_\alpha \subseteq V$ a subset of variables the factor $\psi_\alpha$ depends on, and $Z$ is a normalization constant known as the \textbf{partition function}.

Given a graphical model, we let $\Sigma=\mathcal{X}$ be the set of all possible configurations (variable assignments). Define a weight function $w:\mathcal{X} \rightarrow \mathbb{R}^{+}$ that assigns to each configuration a score proportional to its probability:
$
w(x) = \prod_{\alpha \in \mathcal{I}} \psi_\alpha(\{x\}_\alpha)
$.
$Z$ may then be rewritten as
\begin{equation}
\label{partfunc}
Z = \sum_{x \in \mathcal{X}} w(x) = \sum_{x \in \mathcal{X}} \prod_{\alpha \in \mathcal{I}} \psi_\alpha(\{x\}_\alpha)
\end{equation}
Computing $Z$ is typically intractable because it involves a sum over an exponential number of configurations, and is often the most challenging inference task for many families of graphical models. Computing $Z$ is however needed for many inference and learning tasks, such as evaluating the likelihood of data for a given model, computing marginal probabilities, and parameter estimation~\cite{wainwright2008graphical}.

In the context of graphical models inference, we assume to have access to an optimization oracle that can answer Maximum a Posteriori (MAP) queries, namely, solve the following constrained optimization problem
\[
\arg \max_{x \in \mathcal{X}} p(x \mid \mathcal{C})
\]
that is, we can find the most likely state (and its weight) given some evidence $\mathcal{C}$. This is a strong assumption because MAP inference is known to be an NP-hard problem in general. Notice however that computing $Z$ is a \#P-complete problem, a complexity class believed to be even harder than NP.

\subsection{Quadratures of Integrals}
Suppose we are given a quadrature for a continuous (multidimensional) integral of a function $f:\mathbb{R}^n \rightarrow \mathbb{R}^{+}$ over a high dimensional set $S \subseteq \mathbb{R}^n$
\[
\int_{\mathcal{S}} f(\mathbf{x}) \mathrm{d}\mathbf{x} \approx \sum_{x \in \mathcal{X}} w(x) = W
\]
where $\mathcal{X}$ is some discretization of $S$ (e.g., grid based), and $w(x)$ approximates the integral of $f(\mathbf{x})$ over the corresponding element of volume. In this case, we require a compact representation for $w$ and access to an oracle able to optimize the discretized function, subject to arbitrary constraints. See, e.g., Figure \ref{AlgoVisual}.

For simplicity, in the following we will restrict ourselves to the binary case, i.e., $\Sigma = \mathcal{X} = \{0,1\}^n$. The general multinomial case where the sum is over $\mathcal{X}_1 \times \mathcal{X}_2 \times \cdots \times \mathcal{X}_N$ 
can be transformed into the former case using a binary representation, requiring $\lceil \log_2 |\mathcal{X}_i| \rceil$ bits (binary variables) per dimension $i$.

\section{Preliminaries}
We review some results on the construction and properties of universal hash functions; cf.~\cite{vadhan2011pseudorandomness,goldreich2001randomized}. A reader already familiar with these results may skip to the next section.

\begin{mydef}
\label{pairwisehash}
A family of functions $\mathcal{H} = \{h : \{0,1\}^n \rightarrow \{0,1\}^m\}$ is pairwise independent if the following two conditions hold when $H \leftarrow^R \mathcal{H}$ is a function chosen uniformly at random from $\mathcal{H}$. 1) $\forall x \in \{0,1\}^n$, the random variable $H(x)$ is uniformly distributed in $\{0,1\}^m$. 2) $\forall x_1,x_2 \in \{0,1\}^n$ $x_1 \neq x_2$, the random variables $H(x_1)$ and $H(x_2)$ are independent.
\end{mydef}

A simple way to construct such a function is to think about the family $\mathcal{H}$ of all possible functions $\{0,1\}^n \rightarrow \{0,1\}^m$.  This is a family of not only pairwise independent but \emph{fully} independent functions. However, each function requires $m 2^n$ bits to be represented, and is thus impractical in the typical case where $n$ is large.
On the other hand, \emph{pairwise independent} hash functions can be constructed and represented in a much more compact way as follows; see Appendix for a proof.

%

\begin{proposition}
\label{hashconstruction}
Let $A \in \{0,1\}^{m \times n}$, $b\in \{0,1\}^m$. The family $\mathcal{H}=\{h_{A,b}(x): \{0,1\}^n \rightarrow \{0,1\}^m\}$ where $ h_{A,b}(x)= A x + b \mod{2}$ is a family of pairwise independent hash functions.
\end{proposition}

The space $\mathcal{C} = \{ x: h_{A,b}(x) = p \}$ has a nice geometric interpretation as the translated nullspace of the random matrix $A$. It is therefore a finite dimensional vector space, with operations defined on the field $\mathbb{F}(2)$ (arithmetic modulo $2$). We will refer to constraints in the form $A x = b \mod{2}$ as \textbf{parity constraints}, as they can be rewritten in terms of XORs operations as $A_{i1} x_1 \oplus A_{i2} x_2 \oplus \cdots \oplus A_{in} x_n = b_i$.

\section{The {\tt WISH} Algorithm}
We start with the intuition behind our algorithm to approximate the value of $W$ called \textbf{W}eighted-\textbf{I}ntegrals-And-\textbf{S}ums-By-\textbf{H}ashing (\texttt{WISH}).

\begin{figure*}[tb]
    	\includegraphics[width=0.25\textwidth]{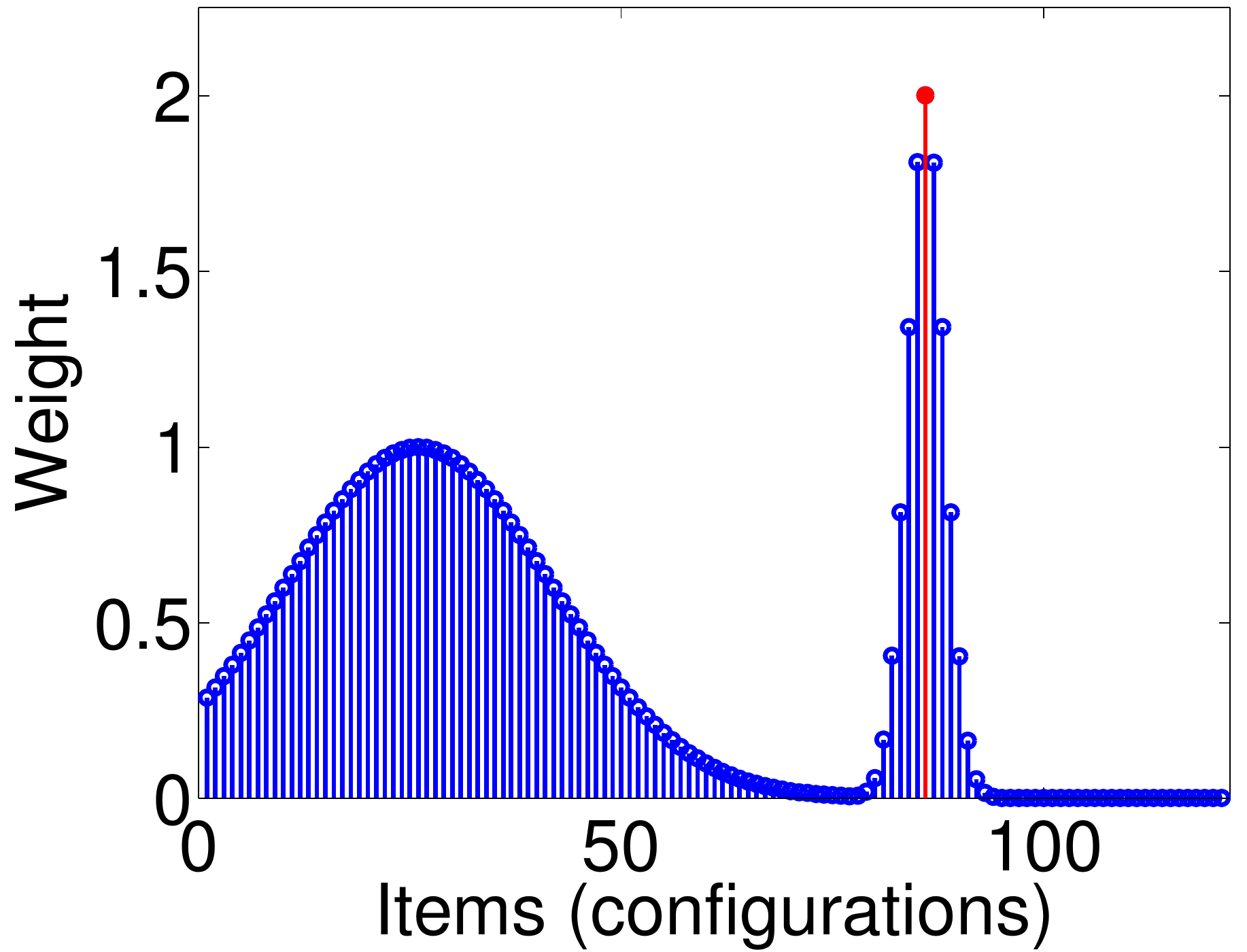}
    	\hfill
    	\includegraphics[trim=40 0 00 0, clip, width=0.23\textwidth]{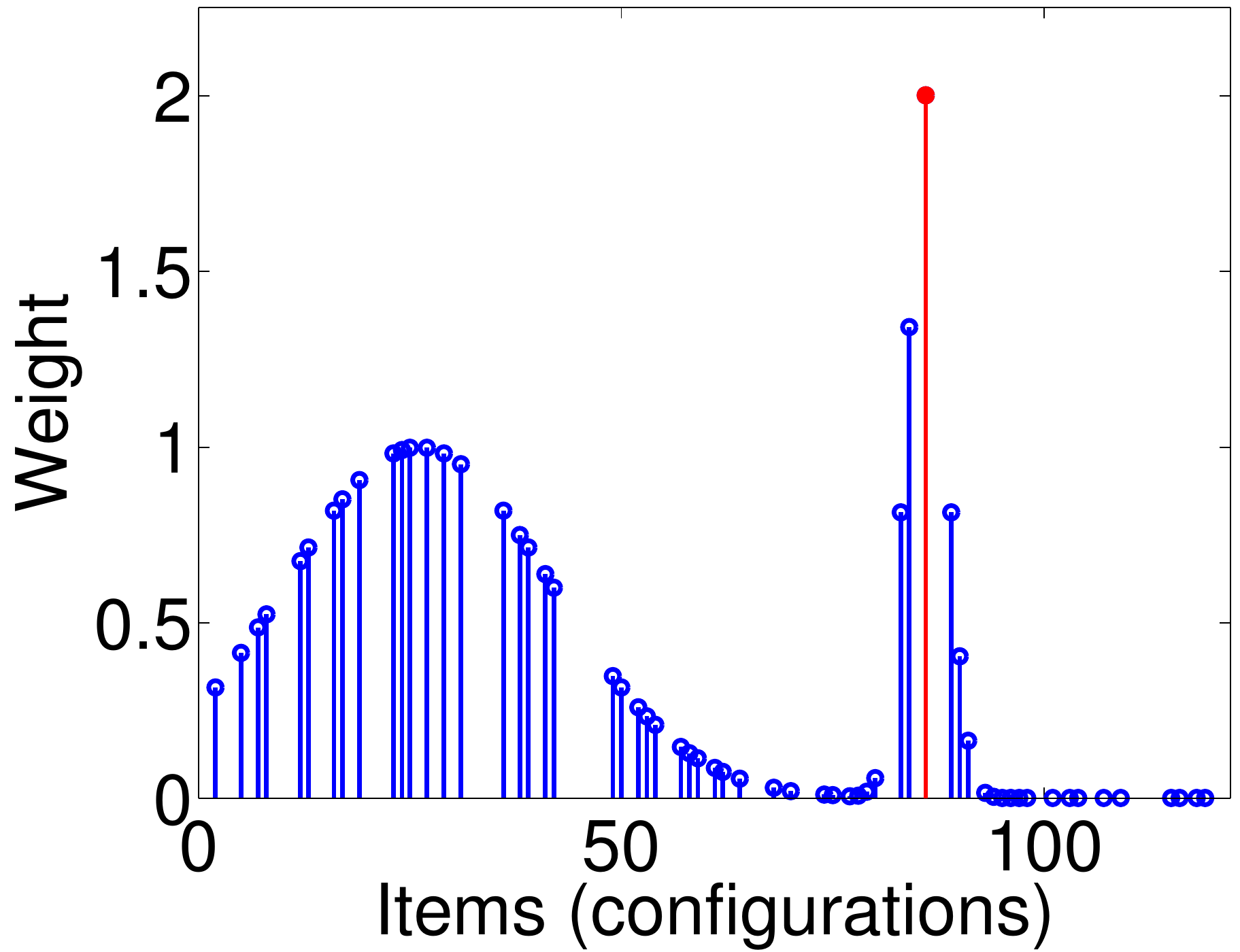}
    	\hfill
	\includegraphics[trim=40 0 00 0, clip, width=0.23\textwidth]{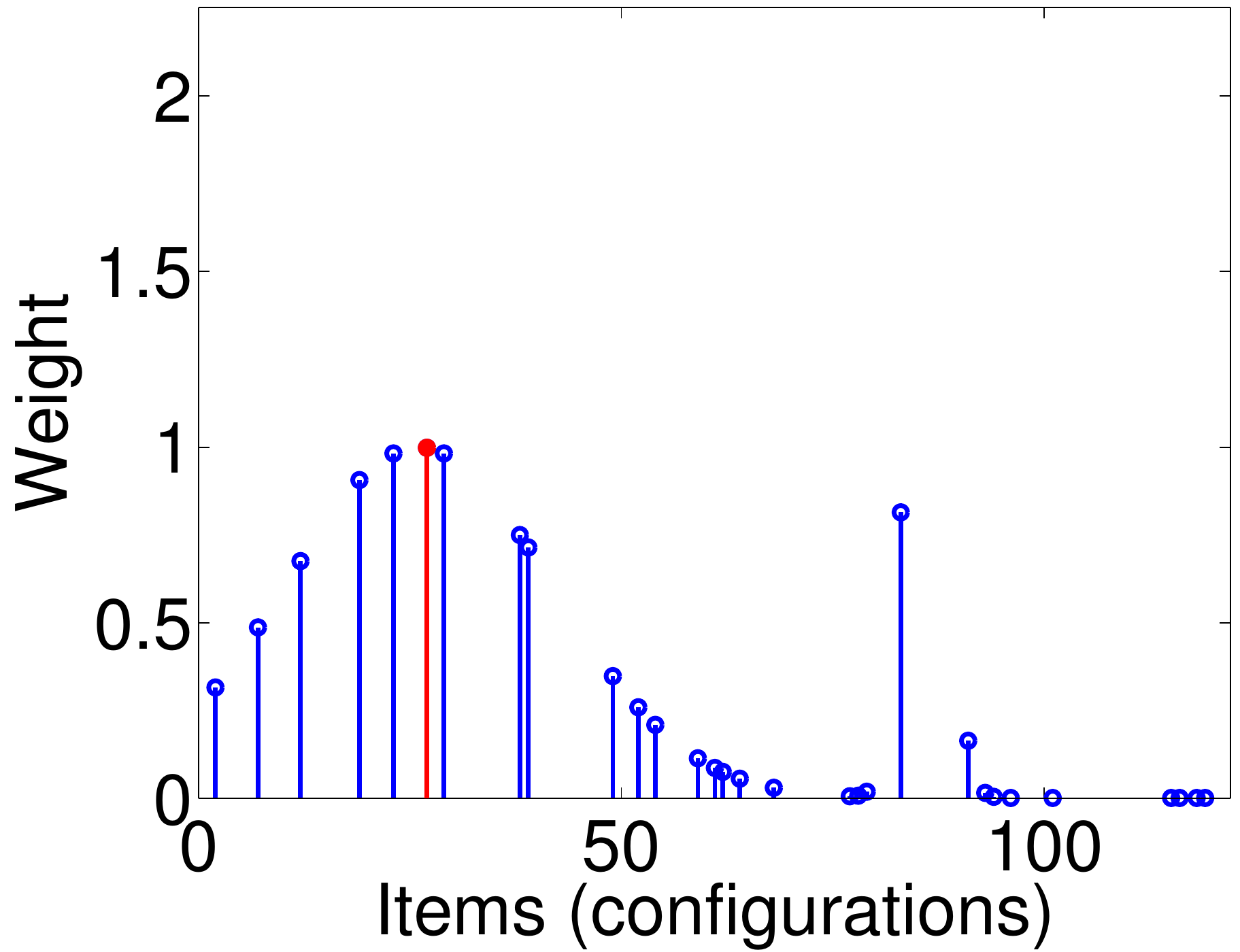}
    	\hfill
   	\includegraphics[trim=40 0 00 0, clip, width=0.23\textwidth]{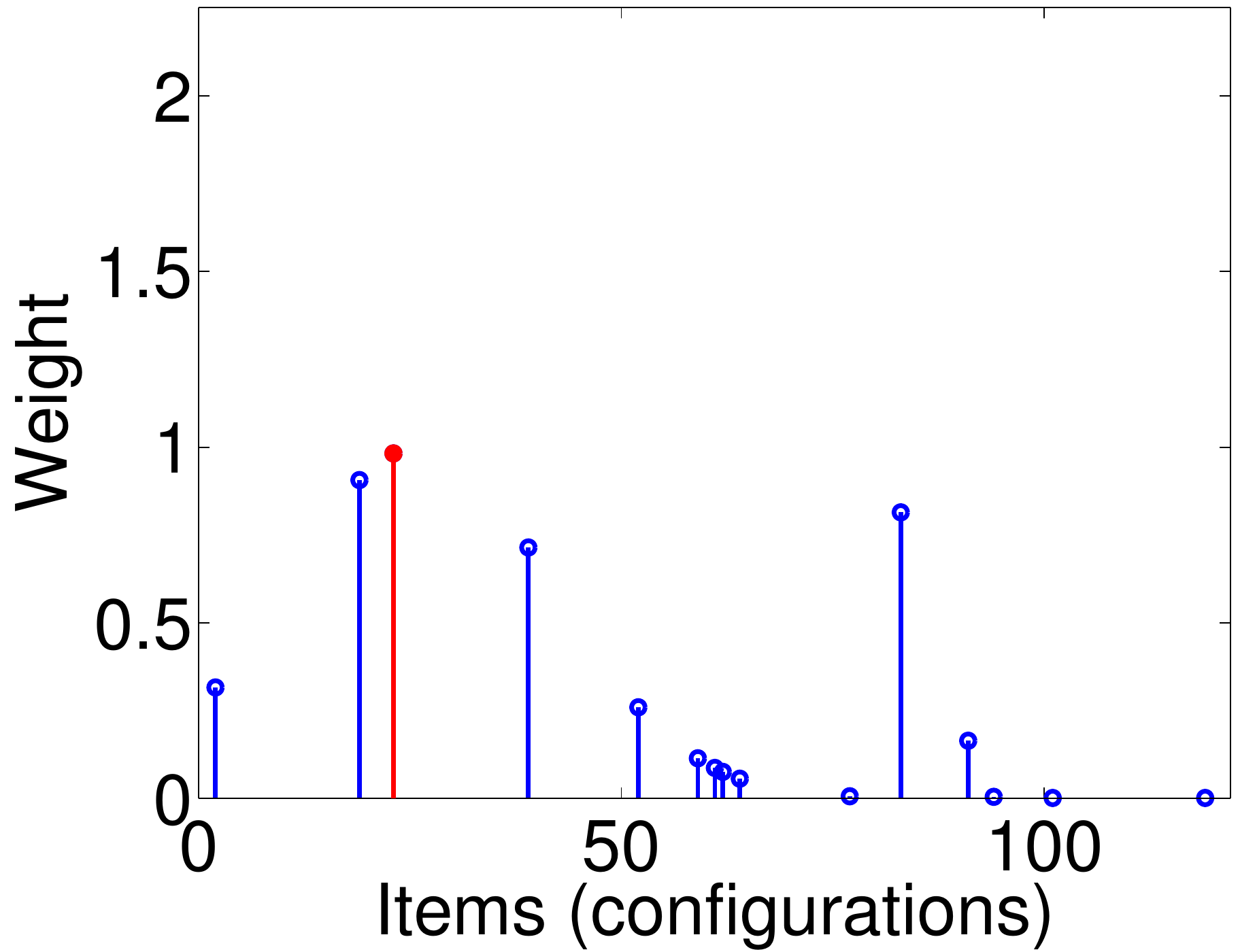}
  	\caption{Visualization of the ``thinning'' effect of random parity constraints, after adding 0, 1, 2, and 3 parity constraints. Leftmost plot shows the original function to integrate. Constrained optimal solution in red.}
     \label{AlgoVisual}
\end{figure*}

Computing $W$ as defined in Equation (\ref{Wdef}) is challenging because the sum is defined over an exponentially large number of items, i.e., $|\Sigma|=2^n$ when there are $n$ binary variables. Let us define the \textbf{tail distribution} of weights as $G(u) \triangleq |\{\sigma \mid w(\sigma) \geq u\}|$. Note that $G$ is a non-increasing step function, changing values at no more than $2^n$ points. Then $W$ may be rewritten as $\int_{\mathbb{R}^+} G(u) \mathrm{d}u$, i.e., the total \emph{area} $A$ under the $G(u)$ vs.\ $u$ curve. One way to approximate $W$ is to (implicitly) divide this area $A$ into either \emph{horizontal} or \emph{vertical} slices (see Figure~\ref{fig:tailcurve}), approximate the area in each slice, and sum up.

Suppose we had an efficient procedure to estimate $G(u)$ given any $u$. Then it is not hard to see that one could create enough slices by dividing up the x-axis, estimate $G(u)$ at these points, and estimate the area $A$ using quadrature. However, the natural way of doing this to any degree of accuracy would require a number of slices that grows at least logarithmically with the weight range on the x-axis, which is undesirable.

Alternatively, one could split the y-axis, i.e., the $G(u)$ value range $[0,2^n]$, at geometrically growing values $1, 2, 4, \cdots, 2^n$, i.e., into bins of sizes $1, 1, 2, 4, \cdots, 2^{n-1}$. Let $b_0 \geq b_1 \geq \cdots \geq b_n$ be the weights of the configurations at the split points. In other words, $b_i$ is the $2^i$-th quantile of the weight distribution. Unfortunately, despite the monotonicity of $G(u)$, the area in the horizontal slice defined by each bin is difficult to bound, as $b_i$ and $b_{i+1}$ could be arbitrarily far from each other. However, the area in the \emph{vertical} slice defined by $b_i$ and $b_{i+1}$ must be bounded between $2^i (b_i - b_{i+1})$ and $2^{i+1} (b_i - b_{i+1})$, i.e., within a factor of 2. Thus, summing over the lower bound for all such slices and the left-most slice, the total area $A$ must be within a factor of 2 of $\sum_{i=0}^{n-1} 2^i (b_i - b_{i+1}) + 2^n b_n = b_0 + \sum_{i=1}^n 2^{i-1} b_i$. Of course, we don't know $b_i$. But if we could approximate each $b_i$ within a factor of $p$, we would get a $2p$-approximation to the area $A$, i.e., to $W$.

\begin{figure}[tb]
    	\hfill
    	\includegraphics[width=0.35\textwidth]{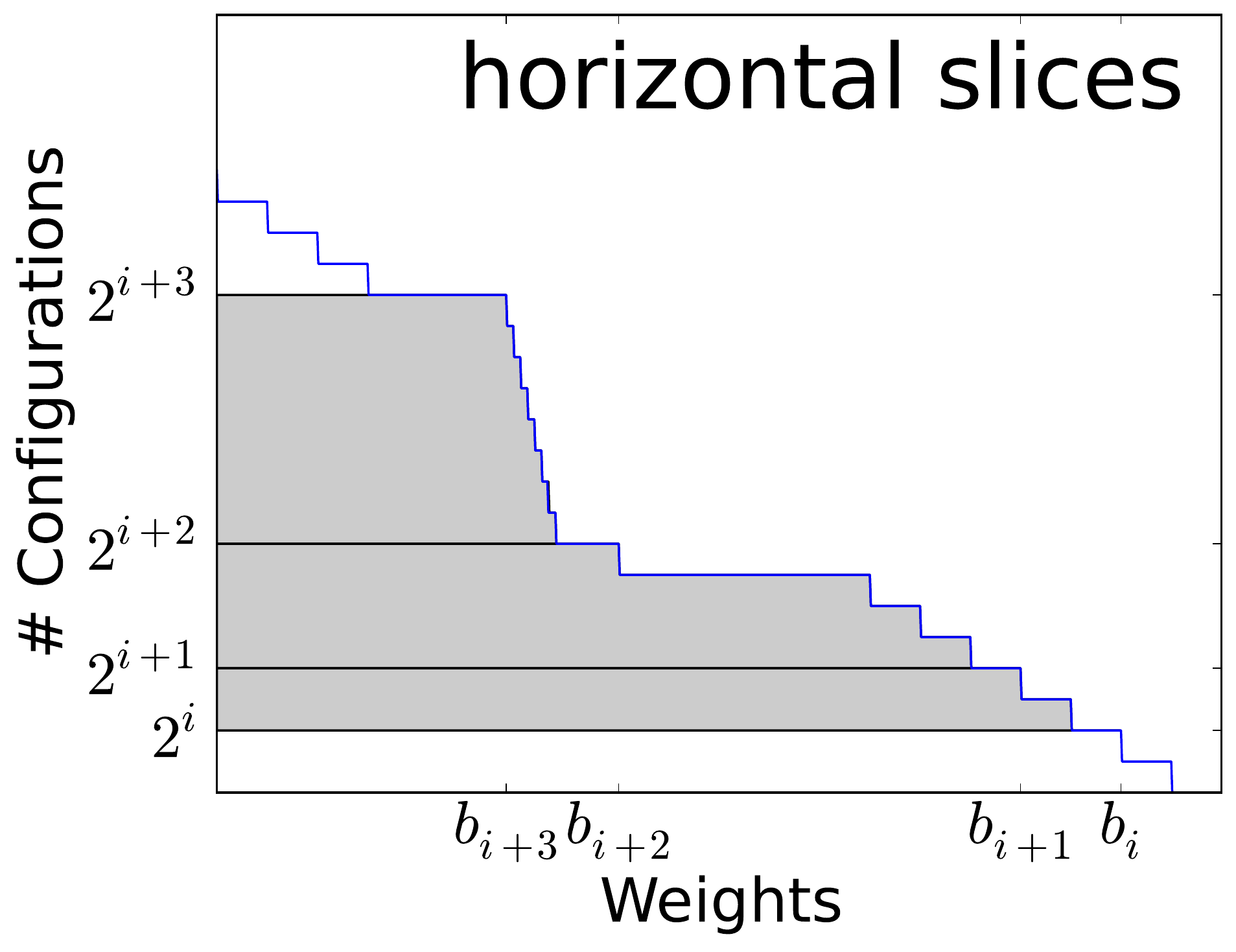}
    	\hfill
    	\includegraphics[width=0.35\textwidth]{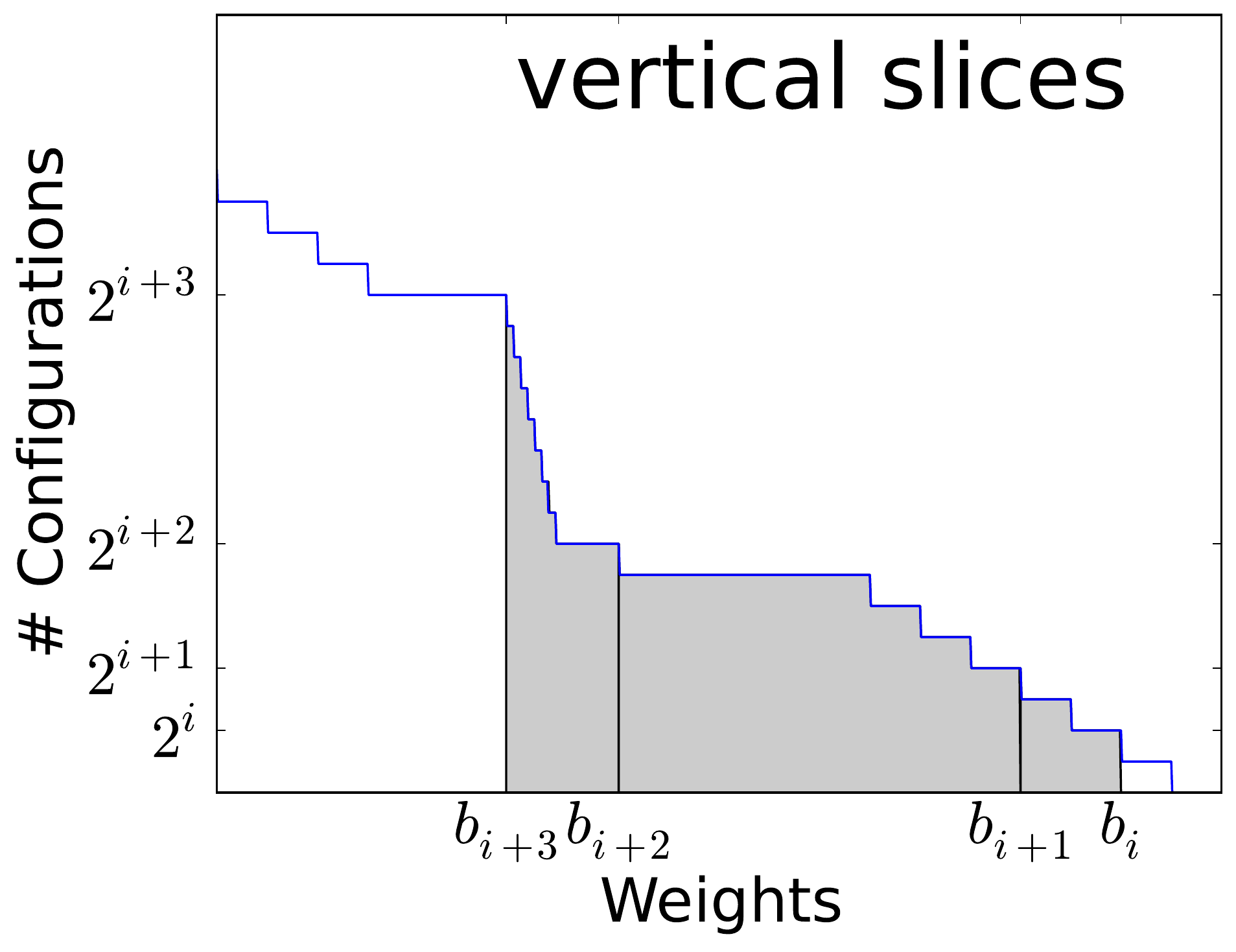}
    	\hfill
     \label{fig:tailcurve}
	\caption{Horizontal vs.\ vertical slices for integration.}
\end{figure}

\texttt{WISH} provides an efficient way to realize this strategy, using a combination of randomized hash functions and an optimization oracle 
to approximate the $b_i$ values with high probability. Note that this method allows us to compute the partition function $W$ (or the area $A$) by estimating weights $b_i$ at $n+1$ carefully chosen points, which is ``only'' an optimization problem.

The key insight to compute the $b_i$ values is as follows. Suppose we apply to configurations in $\Sigma$ a randomly sampled pairwise independent hash function with $2^m$ buckets and use an optimization oracle to compute the weight $w_m$ of a \emph{heaviest} configuration in a fixed (arbitrary) bucket. If we repeat this process $T$ times and consistently find that $w_m \geq w^*$, then we can infer by the properties of hashing that at least $2^m$ configurations (globally) are likely to have weight at least $w^*$. By the same token, if there were in fact at least $2^{m+c}$ configurations of a heavier weight $\hat{w} > w^*$ for some $c > 0$, there is a good chance that the optimization oracle will find $w_m \geq \hat{w}$ and we would not underestimate the weight of the $2^m$-th heaviest configuration. As we will see shortly, this process, using pairwise independent hash functions to keep variance low, allows us to estimate $b_i$ accurately with only $T = \mathrm{O}(\ln n)$ samples.

The pseudocode of \texttt{WISH} is shown as Algorithm~\ref{algo}. It is parameterized by the weight function $w$, the dimensionality $n$, a correctness parameter $\delta > 0$, and a constant $\alpha > 0$. Notice that the algorithm requires solving only $\Theta(n \ln n \ln 1/\delta)$ optimization instances (MAP inference) to compute a sum defined over $2^n$ items. In the following section, we formally prove that the output is a constant factor approximation of $W$ with probability at least $1-\delta$ (probability over the choice of hash functions). Figure~\ref{AlgoVisual} shows the working of the algorithm. As more and more random parity constraints are added in the outer loop of the algorithm (``levels'' increasing from $1$ to $n$), the configuration space is (pairwise-uniformly) thinned out and the optimization oracle selects the heaviest (in red) of the surviving configurations. The final output is a weighted sum over the median of $T$ such modes obtained at each level.


\begin{algorithm}[tb]
\caption{WISH $(w:\Sigma\to\mathbb{R}^+, n = \log_2 |\Sigma|, \delta, \alpha)$ 
}
\label{algo}
\begin{algorithmic}
\STATE $T \leftarrow \left\lceil \frac{\ln \left(1/\delta\right)}{\alpha} \ln n \right\rceil$
\FOR {$i=0, \cdots, n$}
\FOR {$t=1, \cdots, T$}
\STATE Sample hash function $h^i_{A,b}:\Sigma \rightarrow \{0,1\}^i$, i.e. sample uniformly $A \in \{0,1\}^{i\times n}$, $b \in \{0,1\}^i$
\STATE $w_i^t \leftarrow \max_\sigma w(\sigma)$ subject to $A \sigma =b \mod{2}$
\ENDFOR
\STATE $M_i \leftarrow \mathrm{Median}(w_i^1, \cdots, w_i^T)$
\ENDFOR
\STATE Return $M_0 + \sum_{i=0}^{n-1} M_{i+1} 2^i$
\end{algorithmic}
\end{algorithm}

\begin{remark}
The parity constraints  $A \sigma =b \bmod{2}$ do not change the worst-case complexity of an NP-hard optimization problem. Our result is thus consistent with the fact that \#P can be approximated in BPP$^\mathrm{NP}$, that is, one can approximately count the number of solutions with a randomized algorithm and a polynomial number of queries to an NP oracle~\cite{goldreich2001randomized}.
\end{remark}

\begin{remark}
Although the parity constraints we impose are simple linear equations over a field, they can make the optimization harder. For instance, finding a configuration with the smallest Hamming weight satisfying a set of parity constraints is known to be NP-hard, i.e. equivalent to computing the minimum distance of a parity code~\cite{berlekamp1978inherent,Vardy:1997:ACC:258533.258559}. On the other hand, most low density parity check codes can be solved extremely fast in practice using heuristic methods such as message passing.
\end{remark}

\begin{remark}
Each of the optimization instances can be solved independently, allowing natural massive \textbf{parallelization}. We will also discuss how the algorithm can be used in an \textbf{anytime} fashion, and the implications of obtaining suboptimal solutions.
\end{remark}


\section{Analysis}
Since many configurations can have identical weight, it will help for the purposes of the analysis to fix, w.l.o.g., a weight-based ordering of the configurations, and a natural partition of the $|\Sigma| = 2^n$ configurations into $n+1$ bins that the ordering induces.
\begin{mydef}
\label{quantilesdef}
Fix an ordering $\sigma_i, 1 \leq i \leq 2^n,$ of the configurations in $\Sigma$ such that for $1 \leq j < 2^n$, $w(\sigma_j) \geq w(\sigma_{j+1})$. For $i \in \{0,1,\cdots,n\}$, define $b_i \triangleq w(\sigma_{2^i})$. Define a special \emph{bin} $B \triangleq \{\sigma_1\}$ and, for $i \in \{0, 1, \cdots, n-1\}$, define \emph{bin} $B_i \triangleq \{\sigma_{2^i+1}, \sigma_{2^i+2}, \cdots, \sigma_{2^{i+1}}\}$.
\end{mydef}
Note that bin $B_i$ has precisely $2^i$ configurations. Further, for all $\sigma \in B_i$, it follows from the definition of the ordering that $w(\sigma) \in [b_{i+1},b_i]$. This allows us to bound the sum of the weights of configurations in $B_i$ (the ``horizontal'' slices) between $2^i b_{i+1}$ and $2^i b_i$.

\subsection{Estimating the Total Weight}
\label{sec:totalweight}
Our main theorem is that Algorithm~\ref{algo} provides a constant factor approximation to the partition function.

\begin{theorem}
\label{maintheorem}
For any $\delta > 0$ and positive constant $\alpha \leq 0.0042$, Algorithm \ref{algo} makes $\Theta(n \ln n \ln 1/\delta)$ MAP queries and, with probability at least $(1-\delta)$, outputs a 16-approximation of $W = \sum_{\sigma \in \Sigma} w(\sigma)$.
\end{theorem}

The proof relies on two intermediate results whose proofs may be found in the Appendix.

\begin{lemma}
\label{lemma:icdf}
Let $M_i=\mathrm{Median}(w_i^1, \cdots, w_i^T)$ be defined as in Algorithm \ref{algo} and $b_i$ as in Definition \ref{quantilesdef}. Then, for all $c \geq 2$, there exists an $\alpha^*(c) > 0$ such that for $0 < \alpha \leq \alpha^*(c)$,
\[
\Pr\left[ M_i \in [b_{\min\{i+c,n\}},b_{\max\{i-c,0\}}]\right] \geq 1-\exp(-\alpha T)
\]
\end{lemma}

\begin{lemma}
\label{lemma:range-bound}
Let $L' \triangleq b_0 + \sum_{i=0}^{n-1} b_{\min\{i+c+1,n\}}2^i$ and $U' \triangleq b_0 + \sum_{i=0}^{n-1} b_{\max\{i+1-c,0\}}2^i$. Then $U' \leq 2^{2c} L'$.
\end{lemma}

\begin{proof}[Proof of Theorem~\ref{maintheorem}]
%
It is clear from the pseudocode of Algorithm~\ref{algo} that it makes $\Theta(n \ln n \ln 1/\delta)$ MAP queries. For accuracy analysis, we can write $W$ as:
\begin{eqnarray*}
W \triangleq \sum_{j=1}^{2^n} w(\sigma_j) = w(\sigma_1) + \sum_{i=0}^{n-1} \sum_{\sigma \in B_i} w(\sigma) \\
\in \left[b_0 + \sum_{i=0}^{n-1} b_{i+1}2^i, b_0 + \sum_{i=0}^{n-1} b_{i}2^i\right]
\triangleq \left[L,U\right]
\end{eqnarray*}
Note that $U \leq 2 L$ because $2 L = 2 b_0 + \sum_{i=0}^{n-1} b_{i+1}2^{i+1} = 2 b_0 + \sum_{\ell=1}^{n} b_{\ell}2^{\ell} = b_0 + \sum_{\ell=0}^{n} b_{\ell}2^{\ell} \geq U$. Hence, if we had access to the true values of all $b_i$, we could obtain a 2-approximation to $W$.

We do not know true $b_i$ values, but Lemma~\ref{lemma:icdf} shows that the $M_i$ values computed by Algorithm~\ref{algo} are sufficiently close to $b_i$ with high probability. Recall that $M_i$ is the median of MAP values computed by adding $i$ random parity constraints and repeating the process $T$ times. Specifically, for $c \geq 2$, it follows from Lemma~\ref{lemma:icdf} that for $0 < \alpha \leq \alpha^*(c)$,
\begin{align*}
\Pr\left[ \bigcap_{i=0}^{n} \left(M_i \in [b_{\min\{i+c,n\}},b_{\max\{i-c,0\}}] \right) \right] \\
\geq 1-n\exp(-\alpha T) \geq (1-\delta)
\end{align*}
for $T = \frac{\log\left(1/\delta\right)}{\alpha} \log n$, and $M_0 = b_0$.
%
Thus, with probability at least $(1-\delta)$ the output of Algorithm \ref{algo}, $M_0 + \sum_{i=0}^{n-1} M_{i+1} 2^i$, lies in the range:
\begin{eqnarray*}
\left[b_0 + \sum_{i=0}^{n-1} b_{\min\{i+c+1,n\}}2^i ,b_0 +
\sum_{i=0}^{n-1} b_{\max\{i+1-c,0\}}2^i  \right]
\end{eqnarray*}
Let us denote this range $[L',U']$.
By monotonicity of $b_i$, $L' \leq L \leq U \leq U'$. Hence, $W \in [L',U']$.

Applying Lemma~\ref{lemma:range-bound}, we have $U' \leq 2^{2c} L'$, which implies that with probability at least $1-\delta$ the output of Algorithm \ref{algo} is a $2^{2c}$ approximation of $W$. For $c=2$, observing that $\alpha^*(2) \geq 0.0042$ (see proof of Lemma~\ref{lemma:icdf}), we obtain a 16-approximation for $0 < \alpha \leq 0.0042$.
\end{proof}

\subsection{Estimating the Tail Distribution}
We can also estimate the entire tail distribution of the weights, defined as 
$
G(u) \triangleq |\{ \sigma \mid w(\sigma) \geq u\}|
$.
\begin{theorem}
\label{theorem:tail}
Let $M_i$ be defined as in Algorithm~\ref{algo}, $u \in \mathbb{R}^+$, and $q(u)$ be the maximum $i$ such that $\forall j \in \{0, \cdots, i\}, M_j \geq u$. Then, for any $\delta > 0$,
with probability $\geq (1-\delta)$, $2^{q(u)}$ is an 8-approximation of $G(u)$ computed using $\mathrm{O}(n \ln n \ln 1/\delta)$ MAP queries.
\end{theorem}

While this is an interesting result in its own right, if the goal is to estimate the total weight $W$, then the scheme in Section~\ref{sec:totalweight}, requiring a total of only $\Theta(n \ln n \ln 1/\delta)$ MAP queries, is more efficient than first estimating the tail distribution for several values of $u$.

\subsection{Improving the Approximation Factor}

Given a $\kappa$-approximation algorithm such as Algorithm \ref{algo} and any $\epsilon > 0$, we can design a $(1+\epsilon)$-approximation algorithm with the following construction. Let $\ell = \log_{1+\epsilon} \kappa$. Define a new set of configurations $\Sigma^\ell = \Sigma \times \Sigma \times \cdots \times \Sigma$, and a new weight function $w':\Sigma^\ell \rightarrow \mathbb{R}$ as $w'(\sigma_1, \cdots, \sigma_\ell) = w(\sigma_1) w(\sigma_2) \cdots w(\sigma_\ell)$.

\begin{proposition}
Let $\widehat{W}$ be a $\kappa$-approximation of $\sum_{\sigma' \in \Sigma^\ell} w'(\sigma')$. Then $\widehat{W}^{1/\ell}$ is a $\kappa^{1/\ell}$-approximation of $\sum_{\sigma \in \Sigma} w(\sigma)$.
\end{proposition}
To see why this holds, observe that
$
W' = \sum_{\sigma' \in \Sigma^\ell} w'(\sigma') = \left(\sum_{\sigma \in \Sigma} w(\sigma)\right)^\ell = W^\ell
$.
Since $\frac{1}{\kappa} W' \leq \widehat{W} \leq \kappa W'$, we obtain that $\widehat{W}^{1/\ell}$ must be a $\kappa^{1/\ell} = 1+\epsilon$ approximation of $W$.

Note that this construction requires running Algorithm \ref{algo} on an enlarged problem with $\ell$ times more variables. Although the number of optimization queries grows polynomially with $\ell$, increasing the number of variables might significantly increase the runtime.

\subsection{Further Approximations}
When the instances defined in the inner loop are not solved to optimality, Algorithm \ref{algo} still provides approximate \emph{lower bounds} on $W$ with high probability.

\begin{theorem}
\label{approxLB}
Let $\widetilde{w}_i^t$ be suboptimal solutions for the optimization problems in Algorithm \ref{algo}, i.e., $\widetilde{w}_i^t \leq w_i^t$. Let $\widetilde{W}$ be the output of Algorithm \ref{algo} with these suboptimal solutions.
Then, for any $\delta > 0$, with probability at least $1-\delta$, $\frac{\widetilde{W}}{16} \leq W$. 

Further, if $\widetilde{w}_i^t \geq \frac{1}{L}  w_i^t$ for some $L > 0$, then with probability at least $1 - \delta$, $\widetilde{W}$ is a $16 L$-approximation to $W$.
\end{theorem}

The output is always an approximate lower bound, even if the optimization is stopped early. The lower bound is monotonically non-decreasing over time, and is guaranteed to eventually reach within a constant factor of $W$. We thus have an \textbf{anytime} algorithm.

\section{Experimental Evaluation}
We implemented \texttt{WISH} using the open source solver ToulBar2~\cite{toulbar2} to solve the MAP inference problem. ToulBar2 is a complete solver (i.e., given enough time, it will find an optimal solution and provide an optimality certificate), and it was one of the winning algorithms in the UAI-2010 inference competition.
We augmented ToulBar2 with the IBM ILOG CPLEX CP Optimizer 12.3 based techniques borrowed from \citet{gomes2007xors-csp} to efficiently handle the random parity constraints. Specifically, the set of equations $Ax = b \mod{2}$ are linear equations over the field $\mathbb{F}(2)$ and thus allow for efficient propagation and domain filtering using Gaussian Elimination.

For our experiments, we run \texttt{WISH} in parallel using a compute cluster with 642 cores. We assign each optimization instance in the inner loop to one core, and finally process the results when all optimization instances have been solved or have reached a timeout.

For comparison, we consider Tree Reweighted Belief Propagation~\cite{wainwright2003tree} which provides an upper bound on $Z$, Mean Field~\cite{wainwright2008graphical} which provides a lower bound, and Loopy Belief Propagation~\cite{murphy1999loopy} which provides an estimate with no guarantees. We use the implementations of these algorithms available in the LibDAI library~\cite{mooij2010libdai}.

\subsection{Provably Accurate Approximations}

For our first experiment, we consider the problem of computing the partition function, $Z$ (cf.~Eqn.~(\ref{partfunc})), of random Clique-structured Ising models on $n$ binary variables $x_i \in \{0,1\}$ for $i\in\{1,\cdots,n\}$. The interaction between $x_i$ and $x_j$ is defined as $\psi_{ij}(x_i,x_j) =\exp(-w_{ij}) $ when $x_i \neq x_j$, and $1$ otherwise, where $w_{ij}$ is uniformly sampled from $[0,w\sqrt{|i-j|}\,]$ and $w$ is a parameter set to $0.2$. We further inject some structure by introducing a closed chain of strong repulsive interactions uniformly sampled from $[-10 w,0]$. We consider models with $n$ ranging from 10 to 60. These models have treewidth $n$ and can be solved exactly (by brute force) only up to about $n=25$ variables.


Figure \ref{fig:clique.plot} shows the results using various methods for varying problem size. We also computed ground truth for $n \leq 25$ by brute force enumeration. While other methods start to diverge from the ground truth at around $n=25$, our estimate, as predicted by Theorem \ref{maintheorem}, remains very accurate, visually overlapping in the plot. The actual estimation error is much smaller than the worst-case factor of 16 guaranteed by Theorem \ref{maintheorem}, as in practice over- and under-estimation errors tend to cancel out. For $n>25$ we don't have ground truth, but other methods fall \emph{well outside} the provable interval provided by \texttt{WISH}, reported as an error bar that is very small compared to the magnitude of errors made by the other methods.

All optimization instances generated by \texttt{WISH} for $n \leq 60$ were solved (in parallel) to optimality within a timeout of $8$ hours, resulting in high confidence tight approximations of the partition function. We are not aware of any other practical method that can provide such guarantees for counting problems of this size, i.e., a weighted sum defined over $2^{60}$ items.

\subsection{Anytime Usage with Suboptimal Solutions}
Next, we investigate the quality of our results when not all of the optimization instances can be solved to optimality because of timeouts, so that the strong theoretical guarantees of Theorem \ref{maintheorem} do not apply (although Theorem \ref{approxLB} still applies). We consider $10\times10$ binary Grid Ising models, for which ground truth can be computed using the junction tree method~\cite{wainwright2008graphical}. We use the same experimental setup as \citet{hazan2012partition}, who also use random MAP queries to derive bounds (without a tightness guarantee) on the partition function. 
Specifically, we have $n=100$ binary variables $x_i \in \{-1,1\}$ with interaction $\psi_{ij}(x_i,x_j) = \exp(w_{ij} x_i x_j)$. For the attractive case, we draw $w_{ij}$ from $[0,w]$; for the mixed case, from $[-w,w]$. The ``local field'' is $\psi_{ij}(x_i)=\exp(f_i x_i)$ where $f_i$, the strength at site $i$, is sampled uniformly from $[-f,f]$, where $f$ is a parameter with value 0.1 or 1.0.

Figure \ref{GridErrors} reports the estimation \emph{error} for the log-partition function, when using a timeout of $15$ minutes. We see that \texttt{WISH} provides accurate estimates for a wide range of weights, often improving over all other methods. The slight performance drop of \texttt{WISH} for coupling strengths $w \approx 1$ appears to occur because in that weight range the terms corresponding to $i \approx n/2$ parity constraints are the most significant in the output sum $M_0 + \sum_{i=0}^{n-1} M_{i+1} 2^i$. Empirically, optimization instances with roughly $n/2$ parity constraints are often the hardest to solve, resulting in possibly a significant underestimation of the value of $W = Z$ when a timeout occurs.
We do not directly compare with the work of~\citet{hazan2012partition} as we did not have access to their code. However, a visual look at their plots suggests that \texttt{WISH} would provide an improvement in accuracy, although with longer runtime.

\begin{figure*}[htb]
    	\subfigure[Attractive. Field $0.1$.]{\label{gr5fr}\includegraphics[width=0.35\textwidth]{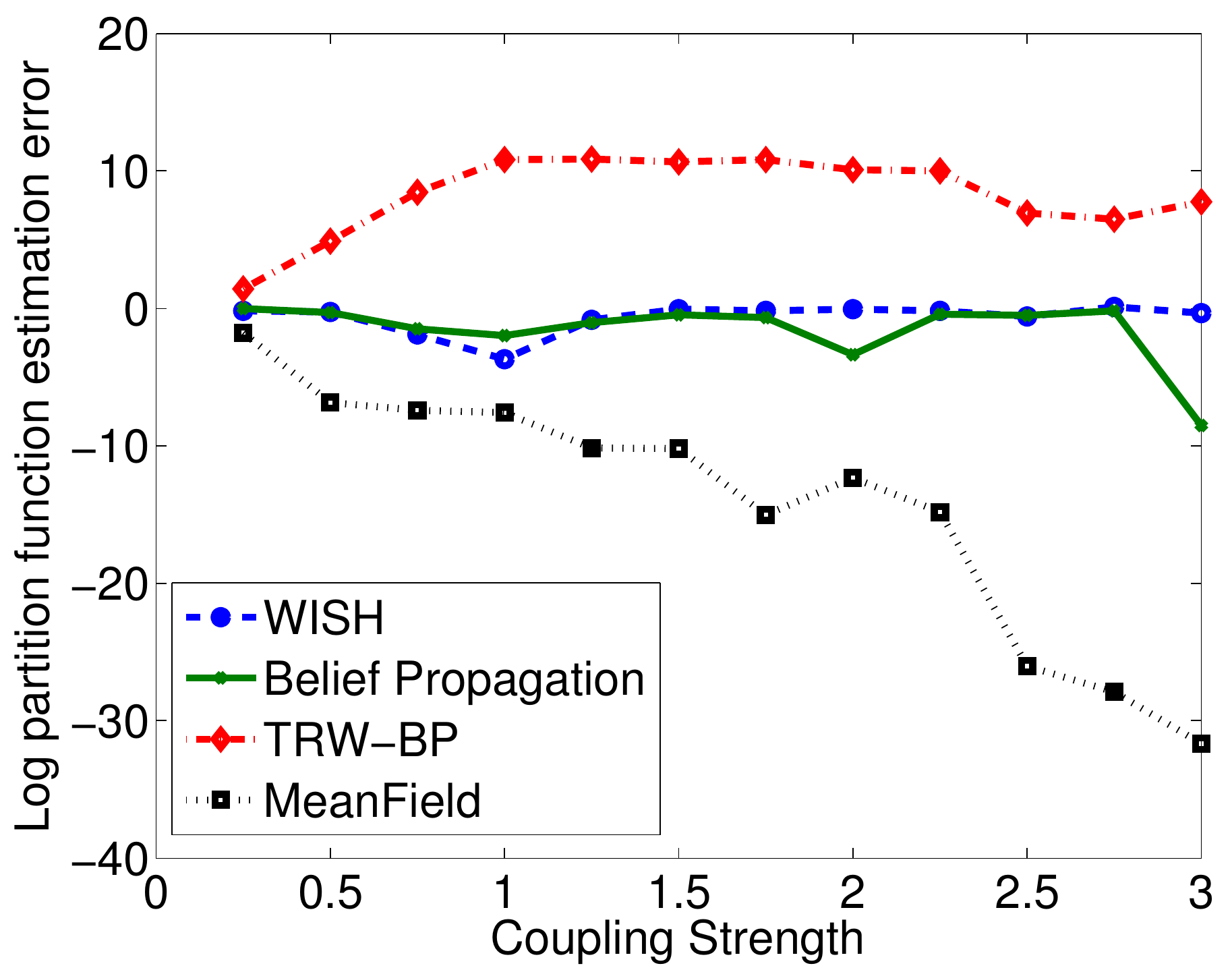}}
    	\hfill
    	\subfigure[Attractive. Field $1.0$.]{\label{g5sp}\includegraphics[trim=25 0 00 0, clip, width=0.35\textwidth]{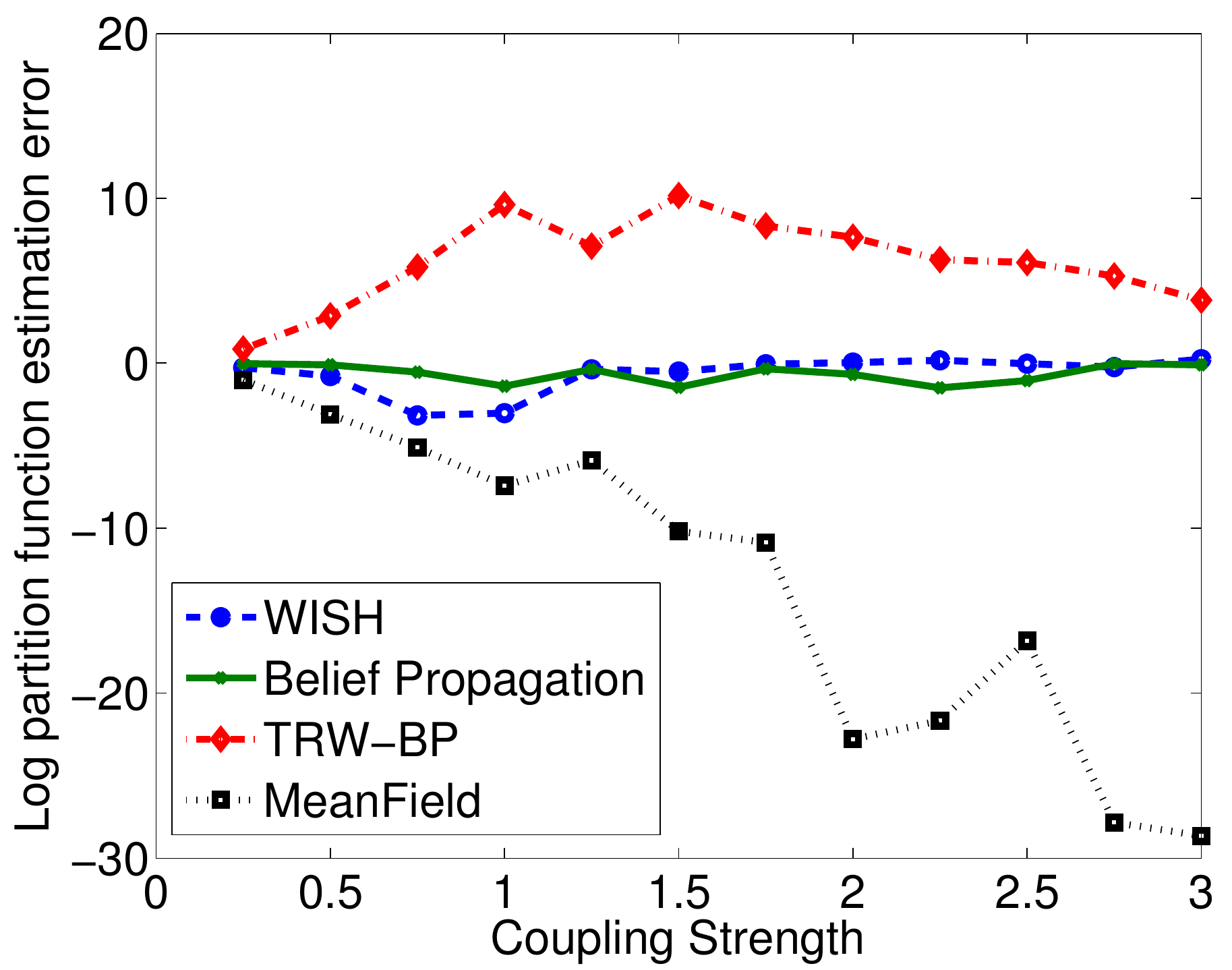}}
    	\\
	\subfigure[Mixed. Field $0.1$.]{\label{cl5fr}\includegraphics[trim=25 0 00 0, clip, width=0.35\textwidth]{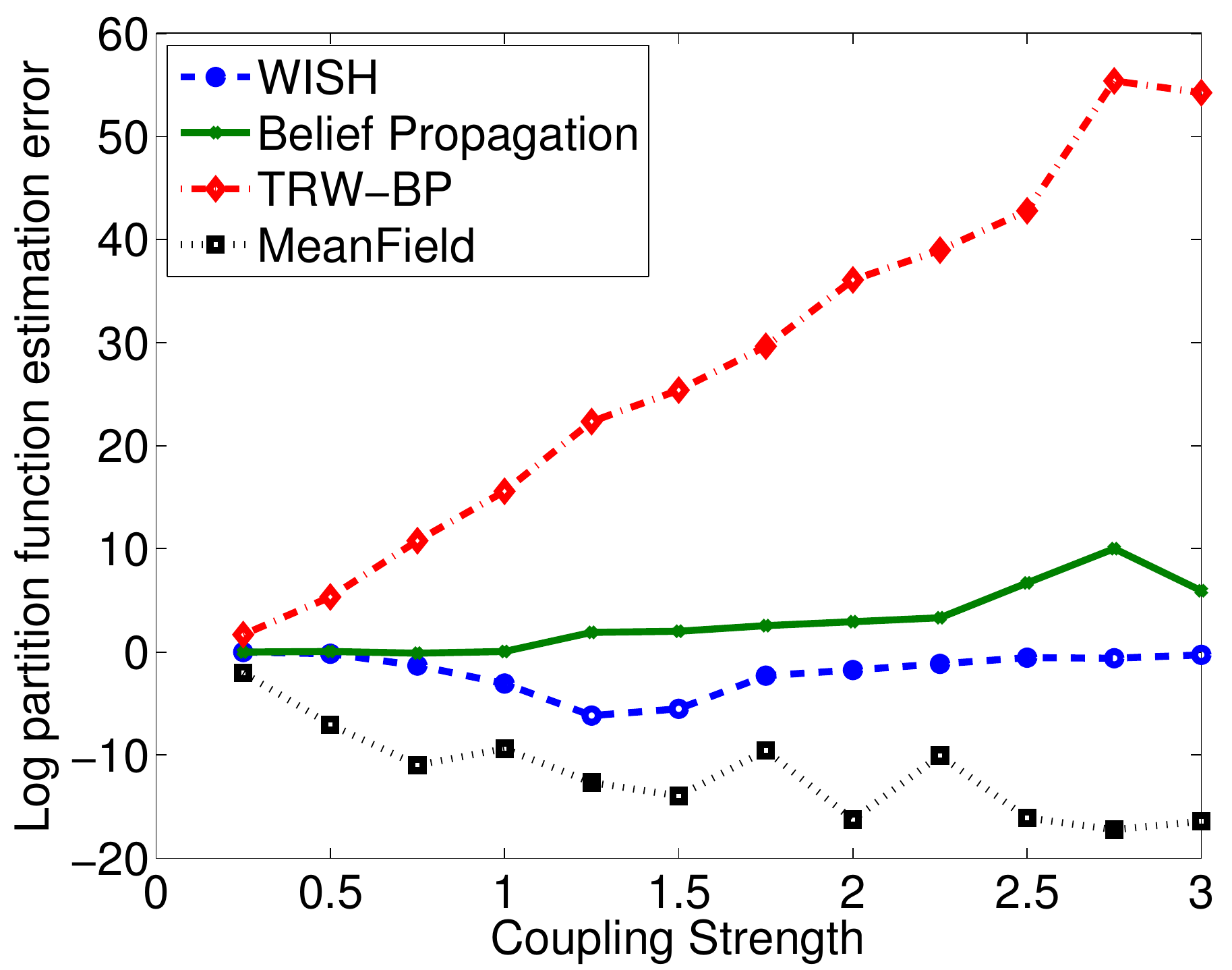}}
    	\hfill
   	\subfigure[Mixed. Field $1.0$.]{\label{cl5sp}\includegraphics[trim=25 0 00 0, clip, width=0.35\textwidth]{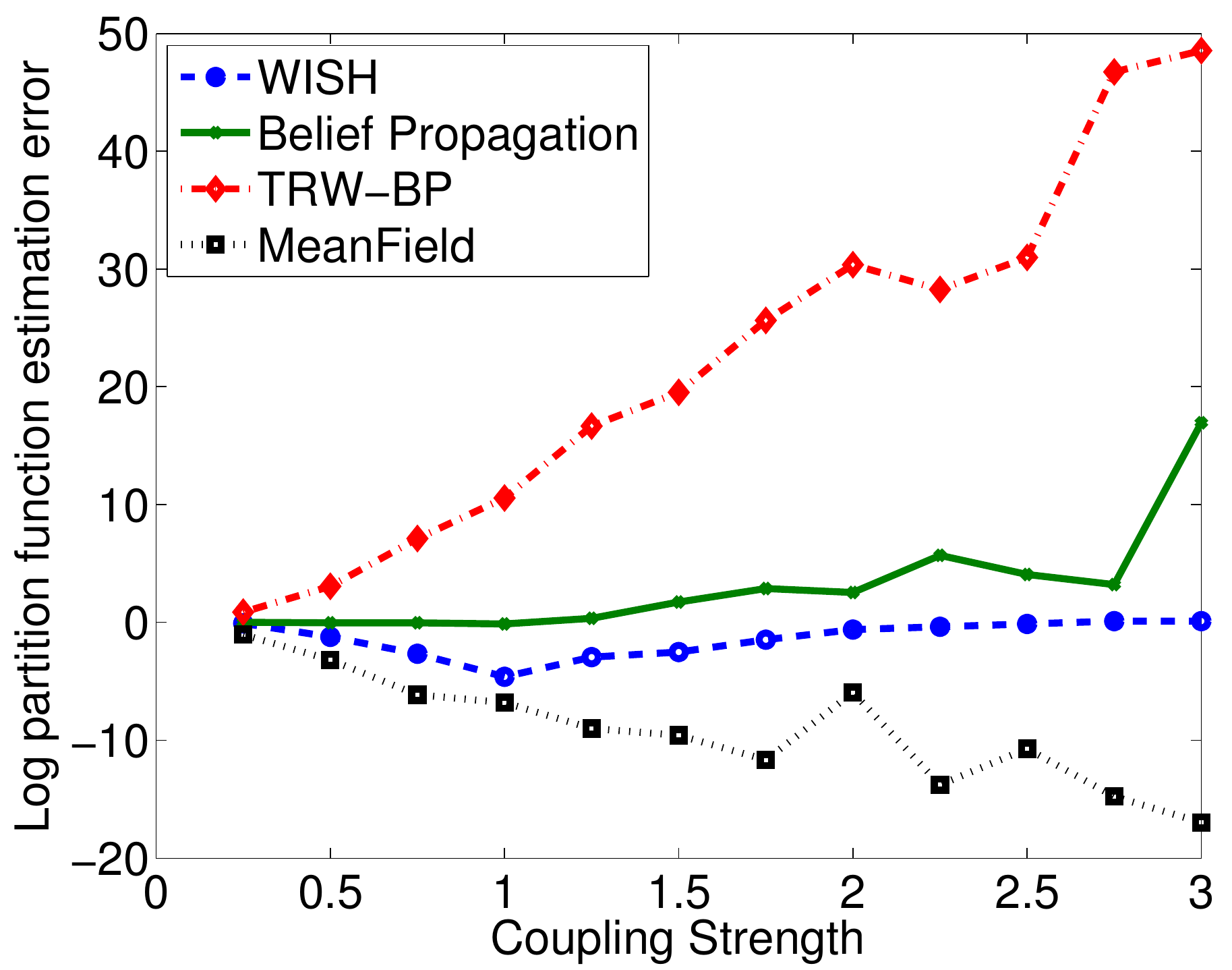}}
  	\caption{Estimation errors for the log-partition function on $10\times10$ randomly generated Ising Grids.}
     \label{GridErrors}
\end{figure*}

\begin{figure*}
\hfill
\subfigure[Log parition function for cliques.]{
		\includegraphics[scale=0.35]{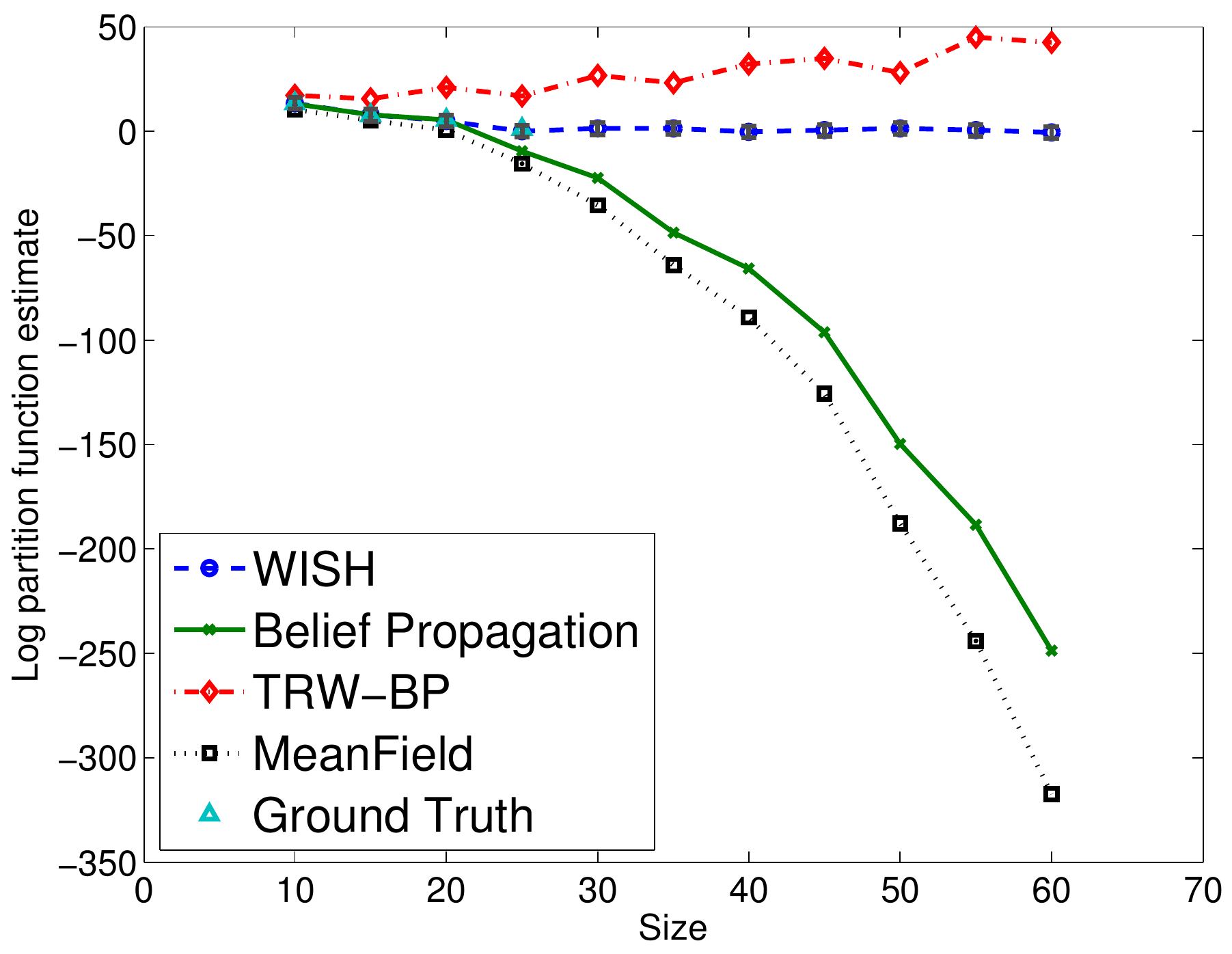}
	\label{fig:clique.plot}
}
\hfill
\subfigure[Sudoku puzzle.]{
\resizebox{0.33\columnwidth}{!}{
\begin{tabular}[b]{!{\vrule width 0.05cm}p{0.1cm}|p{0.1cm}|p{0.1cm}!{\vrule width 0.05cm}p{0.1cm}|p{0.1cm}|p{0.1cm}!{\vrule width 0.05cm}p{0.1cm}|p{0.1cm}|p{0.1cm}!{\vrule width 0.05cm}}
\noalign{\hrule height 0.05cm}
\textcolor{red}{1} & \textcolor{red}{2} & \textcolor{red}{3} & & & & & &\\
\hline
\textcolor{red}{4} & \textcolor{red}{5} & \textcolor{red}{6} & & & & & &\\
\hline
\textcolor{red}{7} & \textcolor{red}{8} & \textcolor{red}{9} & & & & & &\\
\noalign{\hrule height 0.05cm}
 &  &  & & & & & &\\
\hline
 &  &  & & & & & &\\
\hline
 &  &  & & & & & &\\
\noalign{\hrule height 0.05cm}
 &  &  & & & & & &\\
\hline
 &  &  & & & & & &\\
\hline
 &  &  & & & & & &\\
\noalign{\hrule height 0.05cm}
\end{tabular}
}
\label{SudoGrid}
}
\\
\centering
\subfigure[Confabulations from RBM models.]{
		\includegraphics[scale=1.1]{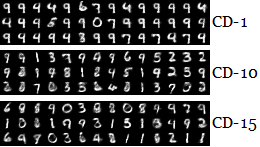}
	\label{SampleDigits}
	}
\caption{Clique Ising models, hard combinatorial structure of Sudoku, and Model Selection for hand-written digits.}
\end{figure*}

\subsection{Hard Combinatorial Structures}
An interesting and combinatorially challenging graphical model arises from Sudoku, which is a popular number-placement puzzle where the goal is to fill a $9 \times 9$ grid (see Figure \ref{SudoGrid}) with digits from $\{1,\cdots,9\}$ so that the entries in each row, column, and $3\times3$ block composing the grid, are all distinct. The puzzle can be encoded as a graphical model with $81$ discrete variables with domain $\{1,\cdots,9\}$, with potentials
$\psi_\alpha(\{x\}_\alpha)=1$ if and only if all variables in $\{x\}_\alpha$ are different, and $\alpha \in \mathcal{I}$ where $\mathcal{I}$ is an index set containing the subsets of variables in each row, column, and block. This defines a uniform probability distribution over all valid complete Sudoku grids (a non-valid grid has probability zero), and the normalization constant $Z_s$ equals the total number of valid grids. It is known that $Z_s = 6.671 \times 10^{21}$. This number was computed exactly with a combination of computer enumeration and clever exploitation of properties of the symmetry group~\cite{felgenhauer2005enumerating}. Here, we attempt to approximately compute this number using the general-purpose scheme \texttt{WISH}.

First, following \citet{felgenhauer2005enumerating}, we simplify the problem by fixing the first block as in Figure \ref{SudoGrid}, obtaining a new problem over $72$ variables whose normalization constant is $Z' = Z_s/9! \approx 2^{54}$. Next, since we are dealing with a feasibility rather than optimization problem, we replace ToulBar2 with CryptoMiniSAT~\cite{soos2009extending}, a SAT solver designed for unweighted cryptographic problems and which natively supports parity constraints. We observed that \texttt{WISH} can consistently find solutions ($60\%$ of the times) after adding $52$ random parity constraints, while for $53$ constraints the success rate drops below $0.5$, at $45\%$. Therefore $M_i = 1$ in Algorithm~\ref{algo} for $i \leq 52$ and there should thus be at least $2^{52} \cdot 9! \approx 1.634 \times 10^{21}$ solutions to the Sudoku puzzle. Although Theorem \ref{maintheorem} cannot be applied due to timeouts for larger values of $i$, this estimate is clearly very close to the known true count. In contrast, the simple ``local reasoning'' done by variational methods is not powerful enough to find even a single solution. Mean Field and Belief Propagation report an estimated solution count of $\exp(-237.921)$ and $\exp(-119.307)$, resp., on a relaxed problem where violating a constraint gives a penalty $\exp(-10)$. 

\subsection{Model Selection}

Many inference and learning tasks require computing the normalization constant of graphical models. For instance, it is needed to evaluate the likelihood of observed data for a given model. This is necessary for Model Selection, i.e., to rank candidate models, or to trigger early stopping during training when the likelihood of a validation set starts to decrease, in order to avoid overfitting~\cite{NIPS2011_1350}.

We train Restricted Boltzmann Machines (RBM) \cite{hinton2006fast} using Contrastive Divergence (CD)~\cite{welling2002new,carreira2005contrastive} on MNIST hand-written digits dataset. In an RBM there is a layer of $n_h$ hidden binary variables $h=h_1, \cdots, h_{n_h}$ and a layer of $n_v$ binary visible units $v=v_1, \cdots, v_{n_v}$. The joint probability distribution is given by
$
P(h,v) = \frac{1}{Z} \exp(b'v+c'h+h'Wv)
$.
We use $n_h=50$ hidden units and $n_v=196$ visible units. We learn the parameters $b,c,W$ using CD-$k$ for $k \in \{1,10,15\}$, where $k$ denotes the number of Gibbs sampling steps used in the inference phase, with $15$ training epochs and minibatches of size $20$.

Figure~\ref{SampleDigits} depicts confabulations (samples generated with Gibbs sampling) from the three learned models. To evaluate the loglikelihood of the data and determine which model is the best, one needs to compute $Z$. We use \texttt{WISH} to estimate this quantity, with a timeout of $10$ minutes, and then rank the models according to the average loglikelihood of the data.
The scores we obtain are $-41.70, -40.35, -40.01$ for $k=1,10,15$, respectively (larger scores means higher likelihood). In this case ToulBar2 was not able to prove optimality for all instances, so only Theorem \ref{approxLB} applies to these results. Although we do not have ground truth, it can be seen that the ranking of the models is consistent with what visually appears closer to a large collection of hand-written digits in Figure \ref{SampleDigits}. Note that $k=1$ is clearly not a good representative, because of the highly uneven distribution of digit occurrences. The ranking of \texttt{WISH} is also consistent with the fact that using more Gibbs sampling steps in the inference phase should provide better gradient estimates and therefore a better learned model. In contrast, Mean Field results in scores $-35.47,-36.08,-36.84$, resp., and would thus rank the models in reverse order of what is visually the most representative order.

\section{Conclusion}
We introduced \texttt{WISH}, a randomized algorithm that, with high probability, gives a constant-factor approximation of a general discrete integral defined over an exponentially large set. \texttt{WISH} reduces the intractable counting problem to a small number of instances of a combinatorial optimization problem subject to parity constraints used as a hash function. In the context of graphical models, we showed how to approximately compute the normalization constant, or partition function, using a small number of MAP queries. Using state-of-the-art combinatorial optimization tools, we are thus able to provide discrete integral or partition function estimates with approximation guarantees at a scale that could till now be handled only heuristically. Finally, our method is a massively parallelizable and anytime algorithm which can also be stopped early to obtain empirically accurate estimates that provide lower bounds with a high probability.


\section*{Acknowledgments}
Supported by NSF Expeditions in Computing grant on
Computational Sustainability \#0832782 and NSF Computing Research Infrastructure
grant \#1059284.

\begin{small}

\end{small}

\appendix
\section{Appendix: Proofs}
\begin{proof}[Proof of Proposition~\ref{hashconstruction}]
Immediately follows from Lemma~\ref{lemma:pairwisehash}.
\end{proof}

\begin{lemma}[pairwise independent hash functions construction]
\label{lemma:pairwisehash}
Let $a \in \{0,1\}^n$, $b\in \{0,1\}$. Then the family $\mathcal{H}=\{h_{a,b}(x): \{0,1\}^n \rightarrow \{0,1\}\}$ where $ h_{a,b}(x)= a \cdot x + b \mod{2}$ is a family of pairwise independent hash functions. The function $h_{a,b}(x)$ can be alternatively rewritten in terms of XORs operations $\oplus$, i.e. $h_{a,b}(x) = a_1 x_1 \oplus a_2 x_2 \oplus \cdots \oplus a_n x_n \oplus b$.
\end{lemma}

\begin{proof}
Uniformity is clear because it is the sum of uniform Bernoulli random variables over the field $\mathbb{F}(2)$ (arithmetic modulo $2$).
For pairwise independence, given any two configurations $x_1,x_2 \in \{0,1\}^n$, consider the sets of indexes $S_1 = \{i : x_1(i)=1\}$, $S_2 = \{i : x_2(i)=1\}$. Then
\begin{align*}
H(x_1) &=& \sum_{i \in S_1 \cap S_2} a_i \oplus \sum_{i \in S_1 \setminus S_2} a_i \oplus b \\
&=& R(S_1 \cap S_2) \oplus R(S_1 \setminus S_2) \oplus b \\
H(x_2) 
&=& R(S_1 \cap S_2) \oplus R(S_2 \setminus S_1) \oplus b
\end{align*}
Note that $R(S_1 \cap S_2)$, $R(S_1 \setminus S_2)$, $R(S_2 \setminus S_1)$ and $b$ are independent as they depend on disjoint subsets of independent variables. When $x_1 \neq x_2$, this implies that $(H(x_1),H(x_2) )$ takes each value in $\{0,1\}^2$ with probability $1/4$.
\end{proof}
As pairwise independent random variables are fundamental tools for derandomization of algorithms, more complicated constructions based larger finite fields generated by a prime power $\mathbb{F}(q^k)$ where $q$ is a prime number are known~\cite{vadhan2011pseudorandomness}. These constructions require a smaller number of random bits as input, and would therefore reduce the variance of our algorithm (which is deterministic except for the randomized hash function use).

\begin{proof}[Proof of Lemma~\ref{lemma:icdf}]
The cases where $i+c>n$ or $i-c<0$ are obvious. For the other cases, let's define the set of the $2^j$ heaviest configurations as in Definition \ref{quantilesdef}:
\[
\mathcal{X}_{j} = \{\sigma_1,\sigma_2, \cdots, \sigma_{2^{j}} \}
\]
Define the following random variable
\[
S_j(h^i_{A,b}) \triangleq \sum_{\sigma \in \mathcal{X}_j}  1_{\{A \sigma =b\bmod{2}\}} 
\]
which gives the number of elements of $\mathcal{X}_{j}$ satisfying $i$ random parity constraints. The randomness is over the choice of $A$ and $b$, which are uniformly sampled in  $\{0,1\}^{i\times n}$ and $\{0,1\}^i$ respectively. By Proposition \ref{hashconstruction}, $h^i_{A,b}:\Sigma \rightarrow \{0,1\}^i$ is sampled from a family of pairwise independent hash functions.
Therefore, from the uniformity property in Definition \ref{pairwisehash}, for any $\sigma$ the random variable $1_{\{A \sigma =b\bmod{2}\}}$ is Bernoulli with probability $1/2^i$. By linearity of expectation,
\[
E [ S_j(h^i_{A,b})] = \frac{|\mathcal{X}_{j}|}{2^i} = \frac{2^j}{2^i}
\]
Further, from the pairwise independence property in Definition \ref{pairwisehash},
\begin{align*}
Var [ S_j(h^i_{A,b})] & =  \sum_{\sigma \in \mathcal{X}_j}  Var\left[1_{\{A \sigma =b\bmod{2}\}}\right] \\
   & = \frac{2^j}{2^i} \left(1 - \frac{1}{2^i}\right)
\end{align*}
Applying Chebychev Inequality, we get that for any $k>0$,
\[
\Pr\left[\left|S_{j}(h^i_{A,b})-\frac{2^j}{2^i}\right|> k \sqrt{\frac{2^j}{2^i} \left(1 - \frac{1}{2^i}\right)}\right] \leq \frac{1}{k^2}
\]
Recall the definition of the random variable $w_i = \max_\sigma w(\sigma)$ subject to $A \sigma =b \mod{2}$ (the randomness is over the choice of $A$ and $b$). Then
\[
\Pr[w_i \geq b_j] = \Pr[w_i \geq w(\sigma_{2^{j}})] \geq \Pr[S_{j}(h^i_{A,b}) \geq 1]
\]
which is the probability that at least one configuration from $\mathcal{X}_{j}$ ``survives'' after adding $i$ parity constraints.

To ensure that the probability bound $1/k^2$ provided by Chebychev Inequality is smaller than a $1/2$, we need $k > \sqrt{2}$. We use $k=3/2$ for the rest of this proof, exploiting the following simple observations which hold for $k=3/2$ and any $c \geq 2$:
\begin{align*}
k \sqrt{2^c} & \leq 2^{c}-1 \\
k \sqrt{2^{-c}} & \leq 1 - 2^{-c}
\end{align*}

For $j = i+c$ and $k$ and $c$ as above, we have that
\begin{eqnarray*}
\Pr[w_i \geq b_{i+c}] \geq \Pr[S_{i+c}(h^i_{A,b}) \geq 1] \geq \\
\Pr\left[|S_{i+c}(h^i)-2^c| \leq 2^c - 1\right] \geq \\
\Pr\left[|S_{i+c}(h^i)-2^c| \leq k \sqrt{2^c}\right] \geq \\
\Pr\left[\left|S_{i+c}(h^i_{A,b})-2^c\right|\leq k \sqrt{2^c \left(1 - \frac{1}{2^i}\right)}\right] \geq \\
1-\frac{1}{k^2} = 5/9 > 1/2
\end{eqnarray*}

Similarly, for $j=i-c$ and $k$ and $c$ as above, we have
$\Pr[w_i \leq b_{i-c}]  \geq 5/9 > 1/2$.

Finally, using Chernoff inequality (since $w_i^1, \cdots, w_i^T$ are i.i.d. realizations of $w_i$)
\begin{align}
\Pr\left[M_i \leq b_{i-c} \right] &\geq 1-\exp(-\alpha'(c) T)\\
\Pr\left[M_i \geq b_{i+c} \right] &\geq 1-\exp(-\alpha'(c) T)
\end{align}
where $\alpha'(2)=2(5/9-1/2)^2$, which gives the desired result
\begin{align*}
\Pr\left[b_{i+c} \leq M_i \leq b_{i-c} \right] & \geq 1-2\exp(\alpha'(c) T)\\
 & = 1-\exp(-\alpha^*(c) T)
\end{align*}
where $\alpha^*(2) = \ln 2 \alpha'(2) = 2(5/9-1/2)^2\ln 2 > 0.0042$
\end{proof}

\begin{proof}[Proof of Lemma~\ref{lemma:range-bound}]
Observe that we may rewrite $L'$ as follows:
\begin{eqnarray*}
L' = b_0 + \sum_{i=n-c-1}^{n-1} b_{n}2^i + \sum_{i=0}^{n-c-2} b_{i+c+1}2^i = \\
b_0 + \sum_{i=n-c-1}^{n-1} b_{n}2^i + \sum_{j=c+1}^{n-1} b_{j}2^{j-c-1}
\end{eqnarray*}
Similarly,
\begin{eqnarray*}
U' = b_0 + \sum_{i=0}^{c-1} b_{0}2^i+ \sum_{i=c}^{n-1} b_{i+1-c}2^i = \\
b_0 + \sum_{i=0}^{c-1} b_{0}2^i+ \sum_{j=1}^{n-c} b_{j}2^{j+c-1} =
  2^c b_0 + 2^c \sum_{j=1}^{n-c} b_{j}2^{j-1} = \\
  2^c b_0 + 2^c \left(\sum_{j=1}^{c} b_{j}2^{j-1} + \sum_{j=c+1}^{n-c} b_{j}2^{j-1}\right) \leq \\
  2^c b_0 + 2^c \left(\sum_{j=1}^{c} b_{0}2^{j-1} + \sum_{j=c+1}^{n-c} b_{j}2^{j-1}\right) = \\
   2^{2c} b_0 + 2^{2c} \sum_{j=c+1}^{n-c} b_{j}2^{j-1-c} \leq \\
  2^{2c} \left(b_0 + \sum_{i=n-c-1}^{n-1} b_{n}2^i + \sum_{j=c+1}^{n-1} b_{j}2^{j-c-1} \right)
  = 2^{2c} L'
\end{eqnarray*}
This finishes the proof.
\end{proof}

\begin{proof}[Proof of Theorem~\ref{theorem:tail}]
As in the proof of Lemma \ref{lemma:icdf}, define the random variable
\[
S_u(h^i_{A,b}) \triangleq \sum_{\sigma \in \{\sigma \mid w(\sigma) \geq u\}}  1_{\{A \sigma =b\bmod{2}\}} 
\]
that gives the number of configurations with weight at least $u$ satisfying $i$ random parity constraints. Then for $i \leq \lfloor \log G(u) \rfloor -c\leq \log G(u) -c$ using Chebychev and Chernoff inequalities as in Lemma \ref{lemma:icdf}
\[
\Pr\left[M_i \geq u \right] \geq 1-\exp(-\alpha' T)
\]
For $i \geq \lceil \log G(u)\rceil +c  \geq \log G(u) +c$, using Chebychev and Chernoff inequalities as in Lemma \ref{lemma:icdf}
\[
\Pr[M_i < u] \geq 1-\exp(-\alpha' T)
\]
Therefore,
\begin{eqnarray*}
\Pr\left[\frac{1}{2^{c+1}} 2^{q(u)} \leq G(u) \leq 2^{c+1} 2^{q(u)}\right] \geq \\
\Pr\left[\bigcap_{i=0}^{\lfloor \log_2 G(u) \rfloor -c} \left(M_i \geq u \right) \bigcap \left(M_{\lceil \log_2 G(u) \rceil +c} < u \right) \right] \geq \\
1-n\exp(-\alpha' T) \geq 1 - \delta
\end{eqnarray*}
This finishes the proof.
\end{proof}

\begin{proof}[Proof of Theorem~\ref{approxLB}]
If  $\widetilde{w}_i^t \leq w_i^t$, from Theorem \ref{maintheorem} with probability at least $1-\delta$ we have $\widetilde{W} \leq M_0 + \sum_{i=0}^{n-1} M_{i+1} 2^i \leq UB'$. Since $\frac{UB'}{2^{2c}}\leq LB' \leq W \leq UB'$, it follows that with probability at least $1-\delta$, $\frac{\widetilde{W}}{2^{2c}} \leq W $.

If $w_i^t \geq \widetilde{w}_i^t \geq \frac{1}{L}  w_i^t$, then from Theorem \ref{maintheorem} with probability at least $1-\delta$ the output is $\frac{1}{L} LB' \leq \widetilde{W} \leq UB'$, and $LB' \leq W \leq UB'$.
\end{proof}

\end{document}